\documentclass[twoside]{article}

\usepackage[accepted]{aistats2026}
\usepackage[round]{natbib}

\usepackage{graphicx} 
\usepackage{amsmath}

\usepackage[utf8]{inputenc} 

\usepackage[english]{babel}

\usepackage[T1]{fontenc}    
\usepackage{hyperref}       
\usepackage{url}            
\usepackage{booktabs}       
\usepackage{amsfonts}       
\usepackage{nicefrac}       
\usepackage{microtype}      
\usepackage{xcolor}         
\usepackage{amsmath}

\usepackage{dsfont}
\usepackage{amsfonts}
\usepackage{bbold}
\usepackage{dirtytalk}
\usepackage{appendix}

\usepackage[textheight=630pt,
  textwidth=468pt,
  centering]{geometry}

\usepackage{amsthm}

\newtheorem{theorem}{Theorem}

\newtheorem{lemma}[theorem]{Lemma}
\newtheorem{proposition}[theorem]{Proposition}
\newtheorem{corollary}[theorem]{Corollary}

\theoremstyle{definition}
\newtheorem{example}[theorem]{Example}

\theoremstyle{remark}
\newtheorem{observation}[theorem]{Observation}
\newtheorem{remark}[theorem]{Remark}

\begin{document}

%

%

\twocolumn[

\aistatstitle{The Riemannian Geometry Associated to Gradient Flows of Linear Convolutional Networks}

\aistatsauthor{ El Mehdi Achour \And Kathlén Kohn \And  Holger Rauhut }

\aistatsaddress{ UM6P College of Computing \\ Rabat, Morocco \And  KTH Royal Inst. of Technology \\ \& Digital Futures, Stockholm, Sweden \And LMU Dept. of Math \& MCML \\ Munich, Germany } ]

\begin{abstract}
  We study geometric properties of the gradient flow for learning deep linear convolutional networks. For linear fully connected networks, it has been shown recently that the corresponding gradient flow on parameter space can be written as a Riemannian gradient flow on function space (i.e., on the product of weight matrices) if the initialization satisfies a so-called balancedness condition. We establish that the gradient flow on parameter space for learning linear convolutional networks can be written as a Riemannian gradient flow on function space regardless of the initialization. This result holds for $D$-dimensional convolutions with $D \geq 2$, and for $D =1$ it holds if all so-called strides of the convolutions are greater than one. The corresponding Riemannian metric depends on the initialization. 
  Our results also apply to shallow ReLU networks.
\end{abstract}

\section{INTRODUCTION}
\label{sec:intro}

Convergence properties of (stochastic) gradient descent schemes for learning deep neural networks are challenging to understand mainly due to the nonconvexity of the corresponding loss functions. In recent years progress could be achieved in specific simplified settings 
\cite{arora2018optimization,ntk,bah2022learning,kohn2022geometry}. One line of work considers deep linear  networks \cite{arora2018optimization,bah2022learning,kohn2022geometry,kohn2024function,ngraul24}. While linear  networks represent linear functions so that their expressivity is  limited, convergence properties of training algorithms are still highly non-trivial and thus interesting to investigate. Such investigations should be seen as a first step before passing to networks with nonlinear activation, as linear networks are common building blocks of nonlinear architectures (e.g., of ReLU networks; see Sec. \ref{sec:relu}). For fully connected linear networks and the square loss it was shown in \cite{bah2022learning} that (essentially) the corresponding gradient flow converges to a global minimizer for almost all initializations; this result was extended to gradient descent under a suitable upper bound on the step sizes \cite{ngraul24}. Additionally, \cite{bah2022learning} revealed that under so-called balanced initialization, the flow of the  network, i.e., the product of the weight matrices, is independent of its parametrization and follows a Riemannian gradient flow with respect to a suitable Riemannian metric. Expressed differently, the so-called neural tangent kernel (determining the Riemannian metric) is independent of the parameterization. This is a weaker property of the neural tangent kernel (NTK) than the well-known one that in the limit of infinite width the NTK becomes constant \cite{ntk} (for random initialization) -- of course, it is then in particular also independent of the parameterization. But in contrast, the mentioned result of \cite{arora2018optimization,bah2022learning} holds for any finite-width network.

In this paper, we study similar geometric and convergence properties for linear \textbf{convolutional} neural networks, i.e., we pass from the fully connected case to certain structured neural networks. More precisely, we study the NTK's sole dependence on the network as a function (instead of its parameters).
Each layer in such a network is a convolution on $D$-dimensional signals that is given by its filter $w_l$, which is a $D$-dimensional tensor, and its stride $s_l \in \mathbb{Z}_{>0}^D$.
Given a $D$-dimensional input tensor, the convolution computes the inner product of various parts of the input tensor with the filter $w_l$, by moving the filter through the input tensor with step size $s_{i,d}$ in the $d$-th direction. 
For a formal definition, see Sections~\ref{sec:oneDim} and~\ref{sec:higherDim}.
For a linear convolutional network with $H$ layers, the following algebraic scalars stay constant throughout gradient flow \cite[Prop. 5.13]{kohn2022geometry}:
\begin{align}
    \label{eq:InvConvolution}
    \delta_l := \Vert w_{l+1} \Vert_F^2 -  \Vert w_{l} \Vert_F^2, \text{ for } l = 1, \ldots, H-1.
\end{align}
If all $\delta_1, \ldots, \delta_{H-1}$ are zero, we say that the filters $(w_1, \ldots, w_H)$ are balanced in analogy to the fully connected case \cite{arora2018optimization,bah2022learning}.
One of our main findings is that in most cases the corresponding NTK and thereby the associated Riemannian metric is independent of the parameterization of the convolutional network, regardless whether the initialization is balanced or not. The NTK only depends on the $\delta_l$ (at initialization).

\subsection{Gradient flow for fully-connected and convolutional linear networks }
\label{sec:learning-networks}

For a fixed neural network architecture, we denote by $\mu: \Theta \to \mathcal{M}$ the network parametrization map that assigns to parameters $\theta \in \Theta$ the end-to-end function $\mu_\theta:= \mu(\theta)$.
The image of $\mu$ is the neuromanifold $\mathcal{M}$ of the network architecture.

We consider both fully connected and convolutional linear neural networks:
\begin{itemize}
    \item \textbf{Fully connected linear neural networks} take the form 
    \[
    \mu_\theta : \mathbb{R}^{d_0} \to \mathbb{R}^{d_H}, \mu_\theta(X) = W_H\cdots W_1 X,
    \]
    where $W_l \in \mathbb{R}^{d_l \times d_{l-1}}$ and $\theta=(W_1,\hdots,W_H)$, i.e, $\mu_{\theta} = W_H \cdots W_1$. Setting $r = \min\{d_0,\hdots,d_H\}$, the neuromanifold $\mathcal{M}$ is the manifold of $d_H \times d_0$ matrices of rank at most $r$.
    \item \textbf{A convolutional linear neural network} is a composition of linear convolutions. On one-dimensional signals, the convolution in the 
    $l$-th layer is given by a filter $w_l \in \mathbb{R}^{k_l}$ and a stride $s_l \in \mathbb{Z}_{>0}$, $l=1,\hdots,H$.
    It is a linear map $\alpha_{w_l,s_l}: \mathbb{R}^{d_{l-1}} \to \mathbb{R}^{d_l}$ with output dimension $d_l$ and input dimension $d_{l-1} = s_l(d_l-1)+k_l$ that sends a vector $X \in \mathbb{R}^{d_{l-1}}$ to the vector with $i$-th entry for  $i = 0, 1, \ldots, d_l-1:$
    \begin{align}\label{def:conv-one-D} \textstyle
    \left(\alpha_{w_l,s_l}(X) \right)_i = \sum_{j=0}^{k_l-1} w_{l,j} X_{i s_l + j}.
    \end{align}
    The network composes $H$ such convolutions, i.e., $f = \alpha_{w_H,s_H} \circ \cdots \circ \alpha_{w_1,s_1}$, which is again a convolution of stride $s_1 \cdots s_H$ and filter size $k := k_1 + \sum_{l=2}^H (k_{l}-1) \prod_{i=1}^{l-1} s_i$.
    Hence,  the network parametrization map 
    \begin{align*}
    \mu: \mathbb{R}^{k_1} \times \ldots \times \mathbb{R}^{k_H} \to \mathbb{R}^k 
    \end{align*}
    assigns to a filter tuple $(w_1, \ldots, w_H)$ the filter of the end-to-end convolution.

    Convolutions and linear convolutional neural networks on $D$-dimensional signals are defined analogously; see \eqref{eq:higherDimensionalConvolution} below.
\end{itemize}

Given training data $\mathcal{D} = \{  (X_i,Y_i) \in \mathbb{R}^{d_0} \times \mathbb{R}^{d_H} \mid i = 1,2, \ldots, N \}$, one aims at learning a hypothesis function $\alpha: \mathbb{R}^{d_0} \to \mathbb{R}^{d_H}$ (e.g. a neural network) 
such that $\alpha(X_i) \approx Y_i$ for $i = 1,\ldots,N$. Given a (differentiable) loss function $\ell :\mathbb{R}^{d_H} \times \mathbb{R}^{d_H} \to \mathbb{R}$, the empirical risk of  $\alpha$ is defined as
\[ \textstyle
\ell_\mathcal{D}(\alpha) = \sum_{i=1}^N \ell(\alpha(X_i),Y_i). 
\]
Below we will exclusively work with the square loss 
\[
\ell(\alpha(X_i),Y_i) = \|\alpha(X_i) - Y_i\|_2^2.
\]
The empirical square loss in parameter space is then 
\[\textstyle
\mathcal{L}_{\mathcal{D}}(\theta) := \ell_{\mathcal{D}} \circ \mu(\theta) = \sum_{i=1}^N \|\mu_\theta(X_i) - Y_i\|_2^2.
\]
Learning via empirical risk minimization requires minimization of $\mathcal{L}_{\mathcal{D}}$. In practice, one often uses (variants of) gradient descent or stochastic gradient descent. For a mathematical analysis it is convenient to view gradient descent as an Euler discretization of gradient flow, which we will consider here exclusively. Introducing a time dependence $\theta(t), t \in [0,\infty)$ and given an initialization $\theta_0$, the gradient flow is the solution to the ODE
\begin{equation}\label{grad-flow-eq}
\frac{d}{dt}\theta(t) = - \nabla \mathcal{L}_{\mathcal{D}}(\theta(t)), \quad \theta(0) = \theta_0.
\end{equation}
The main interest of this paper is to study geometric properties of this gradient flow as well as convergence properties, in particular, whether the flow converges to a minimum of $\mathcal{L}_{\mathcal{D}}$. The \textbf{NTK} introduced in \cite{ntk} is a useful tool to this end. Denoting by $J_\theta$ the Jacobi matrix of $\mu$ at $\theta$, the evolution in function space, i.e., the evolution of $\mu(\theta(t))$ can be written as
\[
\frac{d}{dt} \mu(\theta(t)) = - J_{\theta(t)} J_{\theta(t)}^\top \nabla \ell_{\mathcal{D}}(\mu(\theta(t))), 
\]
where $\nabla \ell_\mathcal{D}(\alpha)$ denotes the (functional) gradient of $\ell_{\mathcal{D}}$ with respect to $\alpha$. The above equation motivates to introduce the NTK as
\[
K_\theta := J_\theta J_\theta^\top. 
\]

In  case of fully connected linear networks where $\mu(\theta) = W_N\cdots W_1$, the gradient flow equation reads
\begin{equation}\label{grad-flow-fully-connected}
\frac{d}{dt} W_j(t) = - \nabla_{W_j} \mathcal{L}_{\mathcal{D}}(W_1(t),\hdots,W_N(t)), 
\end{equation}
for $ j=1,\hdots,H$. It is shown in \cite{arora2018optimization} that  for all $i = 1, \ldots, H-1$,
\begin{align}
\label{eq:balancedFullyConnected}
   \Delta_i := \Delta_i(t) = W_{i+1}^\top(t) W_{i+1}(t) - W_{i}(t)W_i^\top(t)
\end{align}
is constant in $t$, i.e., $\Delta_i(t) = \Delta_i(0)$ for all $t \geq 0$ and all $i=1,\hdots,H-1$. We refer to \cite{marcotte2024abide} for more information on such invariants.
In function space, the corresponding evolution is  
\[
\mu(\theta(t)) = W(t) = W_H(t) \cdots W_1(t).
\]
For balanced initialization, i.e., $\Delta_i(0) = 0$ for all $i=1,\hdots,H-1$, the evolution of $W(t)$ only depends on $W(t)$ itself
but not on the individual factors $W_l(t)$. More precisely, it is shown in \cite{arora2018optimization} that in this case $W(t)$ satisfies the ODE
\begin{align}\label{product-flow}
&\frac{d}{dt} W(t) = \nonumber \\
& - \sum_{j=1}^H (W(t)W(t)^T)^{\frac{H-j}{H}} \nabla \mathcal{L}_{\mathcal{D}}(W(t)) (W(t)^TW(t))^{\frac{j-1}{H}} \nonumber \\
&=: -\mathcal{A}_W(t)\big( \nabla \mathcal{L}_{\mathcal{D}}(W(t))\big).
\end{align}
It is shown in \cite{bah2022learning} that the restriction $\bar{\mathcal{A}}_W$ of $\mathcal{A}_W$ to the tangent space $T_W(\mathcal{M}_r)$ of the manifold $\mathcal{M}_r$ of $d_H \times d_0$ matrices of rank exactly $r$ is invertible. The bilinear form
\[
g_W(Y,Z) = \langle \bar{\mathcal{A}}_W^{-1} Y, Z\rangle, \quad Y,Z \in T_W(\mathcal{M}_r)
\]
generates a Riemannian metric of class $C^1$ on $\mathcal{M}_r$ \cite[Propositions 10 and 11]{bah2022learning}. Moreover, with the associated Riemannian gradient $\nabla^g$ we can write the evolution \eqref{product-flow} of $W(t)$ (in case of balanced initialization, i.e., $\Delta_i(0) = 0$, $i=1,\hdots, H-1$) as Riemannian gradient flow (see also \cite{bah2022learning}):
\[
\frac{d}{dt}W(t) = - \nabla^g \mathcal{L}_{\mathcal{D}}(W(t)).
\]
The NTK at $\theta = (W_1,\hdots,W_H)$ (not necessarily balanced) is given by its action on $Z \in \mathbb{R}^{d_H \times d_H}$ via (see, e.g., \cite[Lemma 2]{bah2022learning}
\begin{align}
\label{NTK-fully-general} \textstyle
K_\theta(Z) = \sum_{l=1}^H  & W_H \cdots W_{l+1} W_{l+1}^\top \cdots W_H^\top \, Z\, \nonumber \\ & W_1^\top \cdots W_{l-1}^\top W_{l-1} \cdots W_1.
\end{align}
If $\theta=(W_1,\hdots,W_H)$ is balanced, i.e., $\Delta_i = 0$, $i=1,\hdots,H-1$, and $W=W_H \cdots W_1$ then 
\[
K_\theta(Z) = \mathcal{A}_{W}(Z);
\]
in particular, the NTK is independent of the parameterization of $W = \mu(\theta)$, i.e., independent of the individual factors $W_l$ in the product $W= W_H \cdots W_1$. Moreover, the Riemannian metric $g_W$ on $\mathcal{M}_k$ is given by the quadratic form induced by $K_\theta^{-1}$ (where the inverse is taken for the restriction on the tangent space $T_W(\mathcal{M}_k)$).

It is shown in \cite{bah2022learning} that for almost all initializations, the gradient flow \eqref{grad-flow-fully-connected} converges to a global minimum of $\ell_{\mathcal{D}}(\alpha)$ restricted to the manifold
$\mathcal{M}_k$ of rank $k$-matrices where $k \leq r := \min\{d_l, l=0,\hdots,H\}$. For $k = r$, this corresponds to the global minimum of $\mathcal{L}_{\mathcal{D}}(\theta)$, but it is currently open whether $k=r$ holds for almost all initializations.

For linear convolutional networks, where $\theta = (w_1,\hdots,w_H)$ and $\mu(\theta)$ is the corresponding end-to-end convolutional filter (taking into account the strides $s_1,\hdots,s_H$), we will study the parameter-independence of NTKs; see Section \ref{sec:mainContr}. Convergence of the corresponding gradient flow for learning linear convolutional neural networks on one-dimensional signals, i.e., $D=1$, is studied in \cite{dirati25}. It is shown that for certain loss functions, including the square loss, the gradient flow \eqref{grad-flow-eq} converges to a critical point of $\mathcal{L}_{\mathcal{D}}$.

\subsection{Main contributions} \label{sec:mainContr}
We prove the following:
\begin{itemize}
    \item For linear convolutional networks on $D$-dimensional signals where $D \geq 2$: For any fixed   $\delta = (\delta_1, \ldots, \delta_{H-1}) \in \mathbb{R}^{H-1}$ and any end-to-end function $v$ in the smooth locus of the neuromanifold $\mathcal{M}$, there are finitely many parameters $\theta = (w_1, \ldots, w_H)$ in $\mu^{-1}(v)$ that satisfy \eqref{eq:InvConvolution}; they all have the same NTK $K_\theta$.  
    Hence, the NTK only depends on the end-to-end function $v$ and the invariants $\delta$.
    \begin{itemize}
        \item For every fixed $\delta$, the NTKs are a Riemannian metric on the smooth locus of $\mathcal{M}$.
        \item For fixed $\delta$, the NTK is generally parameter dependent on the singular locus of $\mathcal{M}$.
    \end{itemize}
    \item For linear convolutional networks on one-dimensional signals where the strides $s_1, \ldots, s_{H-1} $ are larger than $1$:
    All results from the previous bullet points still hold.
    \begin{itemize}
        \item It is known that for enough and sufficiently generic training data, every critical point $\theta$ of the squared-error loss is either the zero function $\mu(\theta)=0$, or $\theta$ is a regular point of $\mu$ and its image $\mu(\theta)$ is in the smooth locus of $\mathcal{M}$ \cite[Theorem~2.12]{kohn2024function}.
        We show that gradient flow with generic initialization does not converge to the zero function. Note that the zero function corresponds to a saddle point of $\mathcal{L}_{\mathcal{D}}$ (unless all $Y_i$ are $0$). As noted above, it has been shown in \cite{dirati25} that (for convolutional networks on one-dimensional signals)  gradient flow always converges to a critical point. This means that altogether
        we can say that we have a metric on the smooth part of $\mathcal{M}$ with convergence to a point in that smooth part, and that this point is a critical point of the loss in both parameter and function space.
    \end{itemize}
    \item For linear convolutional networks on one-dimensional signals where some strides are one: For fixed $\delta$, the NTK is generally parameter dependent.
    \item For fully-connected linear networks, the NTK is generally parameter dependent. While this may intuitively be clear given the available results, this has not been rigorously shown before. In this paper, we therefore provide formal proofs of the following facts.
\begin{itemize}
    \item For generic $\Delta = (\Delta_1, \ldots, \Delta_{H-1}) $, the NTK is generally parameter dependent.
    \item In the balanced case (i.e., all $\Delta_i=0$), we explain the underlying geometry for why the NTK only depends on the function $V$ at hand: The balanced parameters in $\mu^{-1}(V)$ are related by orthogonal matrices, which are isometries of the ambient Euclidean parameter space. 
\end{itemize}
\item Our results also apply to shallow ReLU networks; see Section \ref{sec:relu}.
\end{itemize}

We summarize the key difference between the studied architectures in Table~\ref{tab:table}.

\begin{table}[h]
\begin{tabular}{ll}
\textbf{Linear network}  &\textbf{Is NTK parameter} \\
\textbf{architecture}  &\textbf{independent for} \\
\textbf{}  &\textbf{fixed $\Delta$ resp. $\delta$?} \\
\hline \\
fully-connected         &yes, when balanced  \\
                         &(i.e., $\Delta=0$) \\
1-dim convolutions             &yes, for all $\delta$, \\
                             & when all strides $>1$ \\
high-dim convolutions             &yes, for all $\delta$  \\
                                    & and all strides \\
\end{tabular}
\caption{Linear network architectures studied in this article. Without the assumptions after ``when'', the answer to the question becomes ``generally no''.} 
\label{tab:table}
\end{table}


\begin{remark}\label{rk:map algebraicity}
    For networks with algebraic activation function, the network parametrization map $\mu: \Theta \to \mathcal{M}$ is algebraic, and so is the map that assigns to parameters $\theta \in \Theta$ the NTK $K_\theta$. 
    Therefore, the equality of NTKs $K_{\theta_1}=K_{\theta_2}$ is a polynomial condition in $(\theta_1,\theta_2)$.
    On irreducible algebraic sets, polynomial conditions hold either everywhere or almost nowhere. 
    Since the invariants \eqref{eq:balancedFullyConnected} and \eqref{eq:InvConvolution} are polynomial, 
    the set of parameter pairs $(\theta_1,\theta_2)$ that satisfy the same invariants and give rise to the same function $\mu(\theta_i)$ is an algebraic set. 
    Hence, to prove the bullet items above where ``the NTK is generally parameter dependent'', it is sufficient to exhibit one counterexample. 
\end{remark}

\subsection{Relation to previous works}

\paragraph{Implicit bias and optimization geometry in deep linear networks.} 
Previous works observed that gradient-based methods in deep linear models exhibit non-trivial implicit regularization \cite{saxe2014exact,ji2019gradient,gunasekar2017implicit,gunasekar2018implicit}. For fully-connected architectures with balanced initialization, one part of a geometric explanation for that phenomenon is that gradient flow on parameters induces in fact a Riemannian gradient flow on the function space \cite{arora2019convergence,bah2022learning}. Our results demonstrate that linear CNNs exhibit the latter behavior even without balancedness (except in certain degenerate 1-dimensional cases). This suggests that implicit geometric regularization is more robust in structured convolutional architectures than in fully-connected networks that have been studied in e.g. \cite{moroshko2020implicit,woodworth2020kernel}.

\paragraph{Finite-width NTK behavior.}
Most NTK analyses focus on infinite-width limits, where the kernel becomes stationary \cite{ntk,lee2019wide,arora2019fine}. Finite-width phenomena remain significantly less understood \cite{huang2020dynamics,nguyen2020global}. Our results contribute to this direction by showing that for finite-width CNNs, the NTK depends only on the end-to-end function and the invariants $\delta$, not on internal parameterization. This contrasts sharply with fully connected networks, where parameter dependence persists generically.

\paragraph{Architectural inductive bias and identifiability in structured models.}
Convolutions encode translation-equivariance and local connectivity \cite{lecun2002gradient,cohen2016group,mallat2016understanding}, which are understood as sources of inductive bias \cite{bietti2019inductive}. We show that this structural inductive bias constrains the function space so much that (except in degenerate 1-dimensional cases) generic end-to-end filters admit unique parametrizations (up to scaling). In particular, this leads to a well-defined Riemannian metric on the smooth part of the function space. This aligns with work showing that structural constraints can enforce uniqueness or restrict the representational geometry of deep models \cite{haeffele2015global,sedghi2018singular}, and more broadly with analyses of how architectural structure shapes reachable functions and training dynamics \cite{geiger2020scaling}.

\paragraph{Algebra-geometric approaches to deep learning and their link to optimization.}
Recent work has applied algebraic geometry to analyze parameter symmetries and singularities of neuromanifolds \cite{kileel2019expressive,kohn2024function,finkel2025activation,usevichidentifiability,massarenti2025alexander,arjevani2025geometry,shahverdigeometry,shahverdilearning,marchettiposition}. Also, the critical points have been analyzed via the algebraic viewpoint \cite{mehta2021loss,bharadwaj2023complex,shahverdi2025algebraic,graziani2026stable}. Our paper extends this program by characterizing the parameter symmetries and smooth loci for convolutional neuromanifolds and by linking them to NTK-induced Riemannian metrics. The latter supports the idea that optimization dynamics can be understood in terms of intrinsic geometric data rather than solely statistical or approximation-theoretic properties \cite{bronstein2021geometric,amari1998natural}.

\paragraph{Convergence and saddle avoidance under gradient flow.}
Our results complement recent analyses showing that gradient descent and flow avoid strict saddles under generic conditions \cite{lee2016gradient,panageas2016gradient,bah2022learning}. In particular, for generic data and initialization, we prove that gradient flow converges to a critical point in the smooth locus rather than collapsing to the zero map. This relates to broader work on benign nonconvexity and landscape geometry, which shows that certain deep learning loss landscapes, despite being nonconvex, do not exhibit harmful local minima \cite{DuICLR,Kawaguchi,choromanska15}.


\section{CONVOLUTIONS ON ONE-DIMENSIONAL SIGNALS}
\label{sec:oneDim}

In this section, we focus on linear networks where each layer is a convolution on one-dimensional signals, see \eqref{def:conv-one-D}.
The composition of such convolutions is equivalent to multiplying sparse polynomials as follows.
For positive integers $t$ and $k$, we write $\mathbb{R}[x^t,y^t]_{k-1}$ for the vector space of polynomials that are homogeneous of degree $k-1$ in the variables $(x^t,y^t)$.
This vector space is isomorphic to $\mathbb{R}^k$ via
\begin{align*}
    &\pi_t: \mathbb{R}^k \to \mathbb{R}[x^t,y^t]_{k-1}, \\
    &v \mapsto v_0 x^{t(k-1)} + v_1 x^{t(k-2)} y^t + \ldots  + v_{k-1} y^{t(k-1)}.
\end{align*}
The important insight is the following equality:
\begin{align}
\label{eq:polynomialMultiplication}
    &\pi_1(\mu(w_1, \ldots, w_H)) = \pi_{t_H}(w_H) \cdots \pi_{t_1}(w_1), \\ &\quad \text{ where }
    t_l = s_1 \cdots s_{l-1}. \nonumber
\end{align}
This allows us to view the network parametrization map $\mu$ as multiplying sparse polynomials.
We also see directly from \eqref{eq:polynomialMultiplication} that layers with $k_l=1$ are irrelevant (i.e., such a layer can be omitted) and that the last stride $s_H$ does not affect the map $\mu$. 
Therefore, we assume from now on that $k_l>1$ for all $l = 1, \ldots, H$ and that $s_H = 1$.

The neuromanifold $\mathcal{M} := \mathrm{im}(\mu)$ is a Euclidean closed, semi-algebraic subset of $\mathbb{R}^k$ of dimension $\sum_{l=1}^H (k_l-1) + 1$ \cite[Theorem~2.4]{kohn2024function}. 
If $H>1$, it has singular points. 
This singular locus consists of $0$ and the union of all neuromanifolds that are strictly contained in $\mathcal{M}$ and come from architectures with the same strides $s_1, \ldots, s_H$ \cite[Theorem~2.9]{kohn2024function}. 
This stratification is analogous to the well-known fact that the neuromanifold of a linear fully-connected network consists of all matrices whose rank is upper bounded by some $r$, and that its singular locus (if $r$ is not the maximal rank) is precisely the matrices of rank at most $r-1$. We denote the smooth locus of $\mathcal{M}$ as $\mathrm{Reg}(\mathcal{M})$, which is the set of points in $\mathcal{M}$, which are not singular.

Due to the scaling ambiguity of the filters 
$\mu(\lambda_1 w_1, \ldots, \lambda_H w_H) = \lambda_1 \cdots \lambda_H \mu(w_1, \ldots, w_H)$ (for $\lambda_l \in \mathbb{R}$), 
we can projectify the map $\mu$:
\begin{align*}
    \mu_\mathbb{P}: & \mathbb{P}^{k_1-1} \times \ldots \times \mathbb{P}^{k_H-1} \to \mathbb{P}^{k-1}, 
    \\
    & ([w_1], \ldots, [w_H]) \mapsto [\mu(w_1, \ldots, w_H)].
\end{align*}
Here, we write $[v]$ for the  element in projective space $\mathbb{P}^{k-1}$ corresponding to a nonzero vector $v \in \mathbb{R}^k$.
The image of $\mu_\mathbb{P}$ is the projective neuromanifold $\mathbb{P}\mathcal{M}$.

We see that knowing the invariants in \eqref{eq:InvConvolution} fixes the scaling ambiguity up to signs:
\begin{proposition}\label{prop:correctScaling}
    Let $\delta_1, \ldots, \delta_{H-1} \in \mathbb{R}$, and consider non-zero filters $w_1 \in  \mathbb{R}^{k_1}$, \ldots, $w_H \in \mathbb{R}^{k_H}$.
    Then, there are $2^{H-1}$ scalar tuples $(\lambda_1, \ldots, \lambda_H) \in \mathbb{R}^H$ such that $\lambda_1 \cdots \lambda_H = 1$ and 
    \begin{align*}
        \delta_i = \Vert \lambda_{i+1} w_{i+1} \Vert^2 -  \Vert \lambda_i w_{i} \Vert^2, \ \  \text{ for } i = 1, \ldots, H-1.
    \end{align*}
    These scalar tuples are all equal up to sign changes of the components.
\end{proposition}
The proof of Proposition \ref{prop:correctScaling} and all subsequent proofs can be found in the Appendix.
In the following, we will make use of this  standard fact:
\begin{observation}\label{obs:K}
Since $\mu$ is linear in each of the $w_i$, 
its derivative at a filter tuple $\theta = (w_1, \ldots, w_H)$ is the linear map
\begin{align*}
    d_\theta \mu: \dot \theta = (\dot w_1, \ldots, \dot w_H) \mapsto & \mu(\dot w_1, w_2, \ldots, w_H) + \ldots \\ & + \mu(w_1, \ldots, w_{H-1}, \dot w_H).
\end{align*}
It is represented by the Jacobian matrix $J = J_\theta$.
It has the block columns structure 
$J = [J_1 | J_2 | \cdots | J_H]$, where $J_i$ corresponds to taking the derivative with respect to $w_i$; in other words,  
$$\begin{array}{c c c}
      J_1 \dot w_1 & =&  \mu(\dot w_1, w_2, \ldots, w_H), \\
        \vdots & \vdots & \vdots \\
     J_H \dot w_H & =&  \mu(w_1, \ldots, w_{H-1}, \dot w_H).
\end{array}$$

From the latter expressions, we see that the entries of $J_i$ are linear in the entries of $w_j$ (for $j \neq i$). 
We then conclude $K_\theta = JJ^\top = J_1 J_1^\top + \ldots + J_H J_H^\top$, where the entries of each term $K_i := J_i J_i^\top$ are homogeneous of degree 2 in the entries of $w_j$ (for $j \neq i$).
\end{observation}

\subsection{All strides larger than one}
    We now show that the NTK only depends on the end-to-end filter $v \in \mathcal{M}$ and the invariants $\delta_1, \ldots, \delta_{H-1}$, when the strides in the layers -- up to possibly the last one -- are larger than one.
\begin{theorem}\label{thm:mainStrideLargerOne}
    Let the strides $s_1, \ldots, s_{H-1}$ be larger than one.
    For any fixed   $\delta = (\delta_1, \ldots, \delta_{H-1}) \in \mathbb{R}^{H-1}$ and any end-to-end function $v$ in the smooth locus of the neuromanifold $\mathcal{M}$, there are $2^{H-1}$ many parameters $\theta = (w_1, \ldots, w_H)$ in $\mu^{-1}(v)$ that satisfy \eqref{eq:InvConvolution}; they all have the same NTK $K_\theta$.  
\end{theorem}

Hence, for any fixed choice of invariants $\delta = (\delta_1, \ldots, \delta_{H-1})$, 
the NTK only depends on the smooth point $v \in \mathcal{M}$ and we denote it by $K^{(\delta)}(v)$. See Example \ref{ex:running} in the Appendix.

The NTKs $K^{(\delta)}(v)$ are the inverse of the pushforward of the following Riemannian metric: 
Consider \begin{align*}
    \Theta_\delta := \{ & (w_1, \ldots, w_H) \in  \mu^{-1}(\mathrm{Reg}(\mathcal{M})) \mid \\ &
      \delta_i = \Vert w_{i+1} \Vert^2 - \Vert w_i \Vert^2 \text{ for } i = 1, \ldots, H-1\},
\end{align*}
where we recall that $\mathrm{Reg}(\mathcal{M})$ denotes the smooth locus of $\mathcal{M}$.
This space $\Theta_\delta$ is a smooth manifold of the same dimension as the neuromanifold $\mathcal{M}$.
Its tangent space at any $\theta = (w_1, \ldots, w_H) \in \Theta_\delta$ 
is  
\begin{align}\label{eq:tangentDelta}
    T_\theta \Theta_\delta = \{ &(\dot w_1, \ldots, \dot w_H) \mid \langle w_{i+1},\dot w_{i+1}  \rangle = \langle w_1, \dot w_1 \rangle \nonumber \\ & \text{ for all } i = 1, \ldots, H-1 \},
\end{align}
where $\langle \cdot, \cdot \rangle$ is the standard Euclidean inner product.
\begin{lemma}\label{lem:submersion}
	For every $\theta \in \Theta_\delta$, the derivative $d_{\theta}\mu$ restricted to 
    $T_\theta \Theta_\delta$ is a bijection onto $T_{\mu(\theta)} \mathcal{M}$.
\end{lemma}

\begin{corollary}\label{cor:RiemannianMetric}
	Let the strides $s_1, \ldots, s_{H-1}$ be larger than one.
	The network parametrization map $\mu$ restricted to $\Theta_\delta$ is a submersion onto $\mathrm{Reg}(\mathcal{M})$ with finite fibers.
	The pushforward of the Euclidean metric on $\Theta_\delta$ under $\mu$ is given by the bilinear forms $(K^{(\delta)}(v))^{-1}$ for $v \in \mathrm{Reg}(\mathcal{M})$.
\end{corollary}

In particular, Corollary~\ref{cor:RiemannianMetric} shows that the $K^{(\delta)}(v)$, for any fixed choice of~$\delta$, form a Riemannian metric on the smooth locus of the neuromanifold $\mathcal{M}$. 
This metric can generally not be extended on the singular locus of $\mathcal{M}$.
By \cite{kohn2024function}, the singular locus of $\mathcal{M}$ consists precisely of those end-to-end filters $v$ that either have several filter tuples (even up to scaling by constants) parametrizing it or such that the unique filter tuple $\theta$ (up to scaling) satisfies that the rank $d_\theta \mu$ is less than $\dim(\mathcal{M})$. 
In the latter case, the NTK $K_\theta$ drops rank.
In the former case, the NTKs $K_{\theta_1}$ and $K_{\theta_2}$ for $\mu(\theta_1)=\mu(\theta_2)$ are generally distinct as Example \ref{ex:singular} in the Appendix demonstrates.

Moreover, for varying $\delta$, the metrics $K^{(\delta)}$ on $\mathrm{Reg}(\mathcal{M})$ are generally distinct. See Ex. \ref{ex:varyingDelta} in the Appendix.



\subsection{Convergence of gradient flow}

In this section, we focus on the squared-error loss.
We still assume that the strides in the network layers are larger than one. 
We show that, for a sufficient amount of generic training data and generic initialization, gradient flow  converges to a critical point $\theta$ in parameter space such that $\mu(\theta)$ is a critical point on the smooth locus of the neuromanifold $\mathcal{M}$.

Given training data $\mathcal{D} = \{  (X_i,Y_i) \in \mathbb{R}^{d_0} \times \mathbb{R}^{d_H} \mid i = 1,2, \ldots, N \}$, 
the squared-error loss of an end-to-end convolution $\alpha$ is 
$\ell_\mathcal{D}(\alpha) = \sum_{i=1}^N \Vert \alpha(X_i) - Y_i \Vert^2$.
We denote the squared-error loss in parameter space by $\mathcal{L}_\mathcal{D} :=  \ell_\mathcal{D} \circ \mu$.
Recall that we write $k$ for the size of the end-to-end filters in $\mathcal{M}$.

\begin{theorem}
    \label{thm:convergence}
    Let the strides $s_1, \ldots, s_{H-1}$ be larger than one, and let $N \geq k$.
    For almost all data $\mathcal{D} \in (\mathbb{R}^{d_0} \times \mathbb{R}^{d_H})^N$ and almost all initializations, when minimizing the squared-error loss $\mathcal{L}_\mathcal{D}$, gradient flow  converges to a critical point $\theta$ of $\mathcal{L}_\mathcal{D}$ such that $\mu(\theta)$ is a smooth point of $\mathcal{M}$ and a critical point of $\ell_\mathcal{D} |_{\mathrm{Reg}(\mathcal{M})}$. 
\end{theorem}
\begin{proof}
    It is proven in  
    \cite{dirati25} that gradient flow  converges to a critical point of $\mathcal{L}_\mathcal{D}$.
    Moreover, under the assumptions that the strides are larger than one, $N \geq k$, and sufficiently generic data, \cite[Theorem~2.12]{kohn2024function} shows that every critical point $\theta$ of $\mathcal{L}_\mathcal{D}$ either satisfies $\mu(\theta)=0$ or $\mu(\theta)$ is a smooth point of $\mathcal{M}$ that is a critical point of $\ell_\mathcal{D}$ restricted to that smooth locus.
    Hence, it is enough to exclude $\mu(\theta)=0$ for generic initializations. We explain this separately in Proposition~\ref{prop:exlude0}. 
\end{proof}

\begin{proposition}
    \label{prop:exlude0}
    Let $N \geq k $. For almost all data $\mathcal{D} \in (\mathbb{R}^{d_0} \times \mathbb{R}^{d_H})^N$ and almost all initializations, when minimizing  $\mathcal{L}_\mathcal{D}$, gradient flow  converges to a point $\theta$ with $\mu(\theta)\neq 0$.
\end{proposition}

The proof of Prop.~\ref{prop:exlude0} makes use of \emph{projective duality}, a standard construction from algebraic geometry.

\subsection{Some strides equal to one}

A key fact that enabled Thm. \ref{thm:mainStrideLargerOne} and its proof is that a generic point $v$ in the neuromanifold has a unique point (up to scaling) in its fiber $\mu^{-1}(v)$. 
In other words, the generic non-empty fiber of the projective network parametrization map $\mu_\mathbb{P}$ is a singleton. An algebraic map with that property is called \emph{birational}; meaning that it is bijective almost everywhere.

When some of the strides $s_1, \ldots, s_{H-1}$ are equal to one, the projective network parametrization map $\mu_\mathbb{P}$ is  no longer birational \cite{kohn2024function}.
This means that it is no longer true that a generic point $v$ in the neuromanifold has a unique point (up to scaling) in its fiber $\mu^{-1}(v)$.
In fact, the fibers of $\mu_\mathbb{P} $ are now generally finite but larger than one.
This can be seen via the interpretation  of $\mu$ as polynomial multiplication given in \eqref{eq:polynomialMultiplication}: 
If some stride $s_l$ with $1 \leq l <H$ is equal to one, then $t_l = t_{l+1}$, and so the polynomial factors $\pi_{t_l}(w_l)$ and $\pi_{t_{l+1}}(w_{l+1})$ live in the same ring $\mathbb{R}[x^{t_l},y^{t_l}]$, which means that we can swap some of their irreducible factors without changing the product in \eqref{eq:polynomialMultiplication}. This results in distinct (even up to scaling) filter tuples $\theta_1$ and $\theta_2$ parametrizing the same end-to-end filter, i.e., $\mu(\theta_1) = \mu(\theta_2)$.
They can be scaled such that they satisfy the same invariants \eqref{eq:InvConvolution}, but in general their NTKs remain distinct. See Example \ref{ex:somestridesequaltoone} in the Appendix.

\section{CONVOLUTIONS ON HIGHER-DIMENSIONAL SIGNALS}
\label{sec:higherDim}

A convolution on $D$-dimensional signals is given by a filter $w$ that is a $D$-dimensional tensor of format $k^{(1)} \times \ldots \times k^{(D)}$ and by its stride tuple $\boldsymbol{s} = (s^{(1)}, \ldots, s^{(D)})\in \mathbb{Z}_{>0}^D$.
The convolution is a linear map $\alpha_{w,\boldsymbol{s}}$ that produces an output tensor of format $d^{(1)}\times \ldots \times d^{(D)}$ from an input tensor $X$ of format
$(s^{(1)}(d^{(1)}-1)+k^{(1)}) \times \ldots \times (s^{(D)}(d^{(D)}-1)+k^{(D)})$:
\begin{align}
\label{eq:higherDimensionalConvolution} 
    (\alpha_{w,\boldsymbol{s}}(X))_{i_1, \ldots, i_D}  =& 
    \textstyle \sum_{j_1=0}^{k^{(1)}-1} \cdots 
     \sum_{j_D=0}^{k^{(D)}-1}  
     w_{j_1,\ldots,j_D}  \cdot \nonumber \\ & X_{i_1s^{(1)}+j_1, \ldots, i_Ds^{(D)}+j_D}
\end{align}
for $i_m = 0,1,\ldots,d^{(m)}-1$.
A linear convolutional network composes $H$ such convolutions, which results again in a convolution on $D$-dimensional signals.
Writing $s^{(m)}_l$ and $k^{(m)}_l$ for the strides and filter sizes in the $l$-th layer, the end-to-end convolution
has stride $(s^{(1)}_1 \cdots s^{(1)}_H, \ldots, s^{(D)}_1 \cdots s^{(D)}_H)$ and filter format $k^{(1)} \times \ldots \times k^{(D)} $, where
$k^{(m)} := 
k^{(m)}_1 + \sum_{l=2}^H (k^{(m)}_{l}-1) \prod_{i=1}^{l-1} s^{(m)}_i$. 
The network parametrization map assigns to a tuple of filter tensors $(w_1, \ldots, w_H)$ the filter of the end-to-end convolution $\mu(w_1, \ldots, w_H)$:
\begin{align*}
       \mathbb{R}^{k_1^{(1)} \times \ldots \times k_1^{(D)}} \times \ldots \times \mathbb{R}^{k_H^{(1)} \times \ldots \times k_H^{(D)}} 
    \to
    \mathbb{R}^{k^{(1)} \times \ldots \times k^{(D)}}.
\end{align*}

As in the case of 1-dimensional signals, the end-to-end filter can be equivalently computed via polynomial multiplication. 
Now, we need $D$ pairs of variables $(\boldsymbol{x}, \boldsymbol{y}) := ((x_1,y_1), \ldots, (x_D, y_D))$.
For tuples $\boldsymbol{t} = (t^{(1)}, \ldots, t^{(D)})$
 and $\boldsymbol{k} = (k^{(1)}, \ldots, k^{(D)})$
of positive integers,
we write $\mathbb{R}[(\boldsymbol{x},\boldsymbol{y})^{\boldsymbol{t}}]_{\boldsymbol{k}-\boldsymbol{1}}$ for
the vector space of polynomials that are homogeneous of degree $k^{(m)}-1$ in the variables $(x_m^{t^{(m)}},y_m^{t^{(m)}})$ for $m = 1, \ldots, D$.
We can identify this space of polynomials with  the space of tensors $v$ in $\mathbb{R}^{k^{(1)}\times \ldots\times k^{(D)}}$~via
\begin{align*} 
     \pi_{\boldsymbol{t}}: v \mapsto \sum_{i_1=0}^{k^{(1)}-1} \cdots \sum_{i_D=0}^{k^{(D)}-1}
    & v_{i_1, \ldots, i_D} x_1^{t^{(1)} (k^{(1)}-1-i_1) }
    y_1^{t^{(1)} i_1} \\ &  \cdots x_D^{t^{(D)} (k^{(D)}-1-i_D) }y_D^{t^{(D)} i_D }.
\end{align*}

\begin{lemma} \label{lem:polynomialMultiplication}
    Writing $\boldsymbol{1} := (1, \ldots, 1)$ and $t_l^{(m)} := s_1^{(m)} \cdots s_{l-1}^{(m)}$, we have \begin{align*}
        \pi_{\boldsymbol{1}} (\mu(w_1, \ldots, w_H)) = 
    \pi_{\boldsymbol{t}_H}(w_H) \cdots \pi_{\boldsymbol{t}_1}(w_1).
    \end{align*}
\end{lemma}

The above lemma lets us interpret the network parametrization map $\mu$ as the multiplication of polynomials,
and we can think of the neuromanifold $\mathcal{M}$ as the set of all polynomials 
in the variables $x_1,y_1,\ldots,x_D,y_D$ that can be factored as $Q_H \cdots Q_1$ with $Q_l \in \mathbb{R}[(\boldsymbol{x},\boldsymbol{y})^{\boldsymbol{t}_l}]_{\boldsymbol{k}_l-\boldsymbol{1}}$.
Similarly to the case of one-dimensional signals, we see that a filter size $k^{(m)}_l$ being $1$ means that the variables $x_m,y_m$ do not appear in the $l$-th layer.

\begin{corollary} \label{cor:birationalHighDimensions}
    Let $D >1$, and assume that every layer
    has at least two filter sizes larger than one. 
    Then the projectivization of $\mu$ is birational, up to possibly reordering some of the layer factors.
    In other words, almost every end-to-end filter in the neuromanifold $\mathcal{M}$ has a unique factorization into layer filters, up to scalar multiplication in each layer and possibly reordering some of the layers. 
\end{corollary}

\begin{observation} \label{obs:swap}
\begin{enumerate}
    \item[a)]
    The ordering of the filters in Cor. \ref{cor:birationalHighDimensions} is \emph{not} uniquely determined if and only if the spaces $\mathbb{R}[(\boldsymbol{x},\boldsymbol{y})^{\boldsymbol{t}_l}]_{\boldsymbol{k}_l-\boldsymbol{1}}$ and $\mathbb{R}[(\boldsymbol{x},\boldsymbol{y})^{\boldsymbol{t}_L}]_{\boldsymbol{k}_L-\boldsymbol{1}}$ for two layers $l<L$ coincide.
    In this case,
     \begin{align}\label{eq:swap}
        \pi_{\boldsymbol{t}_l}(w_L) = \pi_{\boldsymbol{t}_L}(w_L), \, \pi_{\boldsymbol{t}_L}(w_l) = \pi_{\boldsymbol{t}_l}(w_l),
    \end{align} and so Lemma~\ref{lem:polynomialMultiplication} implies that we can 
     swap the filters $w_l$ and $w_L$ without affecting the image $\mu(w_1, \ldots, w_H)$. Note that \eqref{eq:swap} requires  $\boldsymbol{t}_l = \boldsymbol{t}_L$, which for $l < L$ can only happen if $D > 1$.
     \item[b)] The NTK is not affected by the swap described in a). 
     To see this, we consider $\theta = (w_1, \ldots, w_H)$ and $\theta'$ that is obtained from  $\theta$ by swapping $w_l$ with $w_L$.
     Observation~\ref{obs:K} also holds for convolutions on higher-dimensional signals as it only depends on the multilinearity of $\mu$.
     So we consider the block structure of the Jacobian matrices $J_\theta = [J_{\theta,1}|\cdots|J_{\theta,H}]$ and $J_{\theta'} = [J_{\theta',1}|\cdots|J_{\theta',H}]$.
     By Observation~\ref{obs:K} and Lemma~\ref{lem:polynomialMultiplication}, the matrix $J_{\theta,i}$ encodes the linear map that multiplies $\pi_{\boldsymbol{t}_i}(\cdot)$ with all polynomials $\pi_{\boldsymbol{t}_j}(w_j)$ for $j \neq i$.
     Therefore, \eqref{eq:swap} implies that $J_{\theta',l} = J_{\theta,L}$, $J_{\theta',L} = J_{\theta,l}$, and $J_{\theta',i} = J_{\theta,i}$ for $i \notin \{l,L \}$.
     Thus, we conclude 
     $K_{\theta'} = \sum_{i=1}^H J_{\theta',i}J_{\theta',i}^\top = \sum_{i=1}^H J_{\theta,i}J_{\theta,i}^\top= K_\theta$.
    \item[c)]
    $\mathbb{R}[(\boldsymbol{x},\boldsymbol{y})^{\boldsymbol{t}_l}]_{\boldsymbol{k}_l-\boldsymbol{1}}$ coincides with $\mathbb{R}[(\boldsymbol{x},\boldsymbol{y})^{\boldsymbol{t}_L}]_{\boldsymbol{k}_L-\boldsymbol{1}}$ if and only if $\boldsymbol{k}_l = \boldsymbol{k}_L$ and, for  all $m = 1, \ldots, D$ with 
    $k^{(m)}_l > 1$, we have $t^{(m)}_l = t^{(m)}_L$.
    The latter condition means that the strides $s^{(m)}_l, \ldots, s^{(m)}_{L-1}$ are equal to one.
\end{enumerate}
\end{observation}

As in the case of one-dimensional signals, the scaling ambiguity that appears in Corollary~\ref{cor:birationalHighDimensions} is fixed (up to signs) by knowing the invariants in \eqref{eq:InvConvolution}.

\begin{proposition}
    \label{prop:correctScalingHighDimension} 
    Let $\delta_1, \ldots, \delta_{H-1} \in \mathbb{R}$, and consider non-zero filter tensors $w_1 \in  \mathbb{R}^{k_1^{(1)} \times \ldots \times k_1^{(D)}}$, \ldots, $w_H \in \mathbb{R}^{k_H^{(1)} \times \ldots \times k_H^{(D)}}$.
    There are $2^{H-1}$ scalar tuples $(\lambda_1, \ldots, \lambda_H) \in \mathbb{R}^H$ such that $\lambda_1 \cdots \lambda_H = 1$ and 
    \begin{align*}
        \delta_i = \Vert \lambda_{i+1} w_{i+1} \Vert_F^2 -  \Vert \lambda_i w_{i} \Vert_F^2, \   \text{ for } i = 1, \ldots, H-1.
    \end{align*}
    These scalar tuples are all equal up to signs. 
\end{proposition}

Putting everything together, we can now conclude for convolutions on higher-dimensional signals that the NTK only depends on the end-to-end filter $v \in \mathcal{M}$ and the invariants $\delta_1, \ldots, \delta_{H-1}$, no matter whether some strides are equal to one.

\begin{theorem}\label{thm:mainHigherDimension}
        Let $D >1$, and assume that every layer
    has at least two filter sizes larger than one.
    For any    $\delta = (\delta_1, \ldots, \delta_{H-1}) \in \mathbb{R}^{H-1}$ and any end-to-end function $v$ in the smooth locus of the neuromanifold $\mathcal{M}$, there are $2^{H-1}$ many parameters $\theta = (w_1, \ldots, w_H)$ in $\mu^{-1}(v)$ that satisfy \eqref{eq:InvConvolution}; they all have the same NTK $K_\theta$. 
    For fixed $\delta$, these NTKs are a Riemannian metric on the smooth locus of $\mathcal{M}$.
\end{theorem}

\section{FULLY-CONNECTED NETWORKS}
\label{sec:fullyConn}

We now move to the case of fully-connected linear networks. As seen in Section~\ref{sec:learning-networks}, the NTK is parameter-independent for balanced initialization.
Moreover, the polynomials  \eqref{eq:balancedFullyConnected} generate \emph{all} algebraic relations that stay constant along the gradient flow curve for generic initialization \citep{marcotte2024abide}. Hence,
a natural question is whether, like for convolutional networks with strides larger than 1, for a generic choice of $\Delta_i$, the NTK is also parameter-independent. While the form \eqref{NTK-fully-general} of the NTK suggests that this is not the case in general, we show this rigorously in Example \ref{ex:fullyconnected} in the Appendix.


We now go back to the balanced case, and provide a geometric reason for why  the NTK is parameter-independent in this case.
The following result shows that the balanced weight matrices that lead to a fixed product matrix $W$ differ only by orthogonal group action. The latter are isometries of the Euclidean metric on the parameter space and thus do not affect the NTK. Moreover, analogously to Corollary \ref{cor:RiemannianMetric}, the inverses of the NTK are the pushforward of the Euclidean metric on the space $\Theta_0$ of balanced matrix tuples under the network parametrization map $\mu$, which is simply the quotient metric under the orthogonal action.
\begin{theorem}\label{thm:fullyconnected}
    Let $(W_H,\ldots,W_1) \in GL_d(\mathbb{R})^H$ such that $W=W_H \cdots W_1$, and $\Delta_i = W_{i+1}^\top W_{i+1} - W_{i}W_i^\top=0$ for all $i \in \{1,\ldots,H-1\}$. Then the set of balanced weights with product  $W$  is equal to
    \begin{align*}
        &  \mu^{-1}(W) \cap \Theta_0 = \left\{  (W_HG_{H-1},G_{H-1}^{T}W_{H-1}G_{H-2}, \right. \\ & \left. \ldots,   G_{2}^{T}W_{2}G_{1},   G_{1}^{T}W_{1}) \mbox{ s.t. } G_{H-1},\ldots,G_1 \in \mathcal{O}_d(\mathbb{R}) \right\}.
    \end{align*}
\end{theorem}

\section{EXTENSION TO RELU NETWORKS} \label{sec:relu}
We illustrate in the case of shallow (i.e., two-layer) networks how the results presented in this work can be extended from linear to ReLU networks. 

\paragraph{Shallow ReLU CNN.} The NTK depends only on the end-to-end function, which we see as follows: The first observation is that the first convolutional layer is either zero or a linear map of full rank \cite[Lem. 4.1]{shahverdigeometry}. Hence, for a nonzero network, there is a (Euclidean) open subset in the domain of the network where its first layer maps to the positive orthant, and so the network restricted to that domain subset is simply a shallow linear CNN. Hence, our factorization results on linear CNNs also apply to this ReLU case. In particular, for higher-dimensional convolutions, the filters in the two layers are generically uniquely identifiable (up to scaling) from the end-to-end function (as in Cor. \ref{cor:birationalHighDimensions}). The second observation is that the $\delta$ invariants are still valid in the ReLU case, and - as in the linear case - they fix the scalings of the filters up to sign (as in Prop. \ref{prop:correctScalingHighDimension}). Finally, it is easy to see that the sign of the first layer cannot be flipped without changing the end-to-end function. All in all (except in degenerate cases of dimension 1 or stride 1), for a nonzero shallow ReLU CNN, there is only a single choice of parameters (i.e., filters) yielding the same end-to-end function. Hence, the NTK (trivially) only depends on the end-to-end~function.

 \paragraph{Shallow ReLU MLP.} The  argument above does not apply to fully-connected networks (MLPs). Already the first observation fails: Since fully-connected layers can have low rank, their range can completely miss the positive orthant. In fact, the study of parameter symmetries in ReLU MLPs is very subtle and not completely settled \cite{grigsby23a,ELISENDAGRIGSBY2025110636}. Either way, the NTK of ReLU MLPs is typically not parameter independent, even for gradient-flow  invariants. To see this, we  consider the shallow network in Example \ref{ex:fullyconnected}. By \cite{marcotte2024abide}, the conserved quantities via gradient flow for shallow ReLU networks are a subset of the linear ones, namely $\mathrm{diag}(W_2^\top W_2 - W_1W_1^\top)$. Therefore, if we restrict the input space to positive $X$, then -- as in the linear case -- both $U_2,U_1$ and $V_2,V_1$ 
 from Example \ref{ex:fullyconnected} give rise to the same end-to-end function while also having equal conserved quantities. We have seen in Example \ref{ex:fullyconnected} that they lead to different values for the linear NTK, and thus also their ReLU NTKs are different.

\section{CONCLUSION}
In this paper, we investigated how gradient flow on parameter space influences the dynamics of the end-to-end function in linear networks. We showed that, for a broad class of linear convolutional networks, the NTK -- and thus the corresponding function-space ODE -- can be expressed independently of the individual layer parameters, relying instead on a global quantity defined at initialization, leading to a Riemaniann gradient flow.
In contrast, we demonstrated that this property does not generally extend to fully connected linear networks, highlighting a fundamental structural difference between the two architectures. 
We have further explained how these results on deep linear networks extend to shallow ReLU networks. 
Our findings raise intriguing questions about whether similar principles might emerge in the context of more practical, nonlinear networks.

\section*{Acknowledgements}
We thank Antonio Lerario and Giovanni Luca Marchetti for helpful discussions.
 KK was supported by the Wallenberg AI, Autonomous Systems and Software Program (WASP) funded by the Knut and Alice Wallenberg Foundation.
 EMA and HR acknowledge funding by the Deutsche Forschungsgemeinschaft (DFG, German Research Foundation) - Project number 442047500 through the Collaborative Research Center ``Sparsity and Singular Structures'' (SFB 1481). EMA acknowledges funding by UM6P as well.

\bibliography{refs}
\bibliographystyle{plainnat}

\newpage 

\section*{Checklist}

The checklist follows the references. For each question, choose your answer from the three possible options: Yes, No, Not Applicable.  You are encouraged to include a justification to your answer, either by referencing the appropriate section of your paper or providing a brief inline description (1-2 sentences). 
Please do not modify the questions.  Note that the Checklist section does not count towards the page limit. Not including the checklist in the first submission won't result in desk rejection, although in such case we will ask you to upload it during the author response period and include it in camera ready (if accepted).

\begin{enumerate}

  \item For all models and algorithms presented, check if you include:
  \begin{enumerate}
    \item A clear description of the mathematical setting, assumptions, algorithm, and/or model. [Yes]
    \item An analysis of the properties and complexity (time, space, sample size) of any algorithm. [Not Applicable]
    \item (Optional) Anonymized source code, with specification of all dependencies, including external libraries. [Not Applicable]
  \end{enumerate}

  \item For any theoretical claim, check if you include:
  \begin{enumerate}
    \item Statements of the full set of assumptions of all theoretical results. [Yes]
    \item Complete proofs of all theoretical results. [Yes]
    \item Clear explanations of any assumptions. [Yes]     
  \end{enumerate}

  \item For all figures and tables that present empirical results, check if you include:
  \begin{enumerate}
    \item The code, data, and instructions needed to reproduce the main experimental results (either in the supplemental material or as a URL). [Not Applicable]
    \item All the training details (e.g., data splits, hyperparameters, how they were chosen). [Not Applicable]
    \item A clear definition of the specific measure or statistics and error bars (e.g., with respect to the random seed after running experiments multiple times). [Not Applicable]
    \item A description of the computing infrastructure used. (e.g., type of GPUs, internal cluster, or cloud provider). [Not Applicable]
  \end{enumerate}

  \item If you are using existing assets (e.g., code, data, models) or curating/releasing new assets, check if you include:
  \begin{enumerate}
    \item Citations of the creator If your work uses existing assets. [Not Applicable]
    \item The license information of the assets, if applicable. [Not Applicable]
    \item New assets either in the supplemental material or as a URL, if applicable. [Not Applicable]
    \item Information about consent from data providers/curators. [Not Applicable]
    \item Discussion of sensible content if applicable, e.g., personally identifiable information or offensive content. [Not Applicable]
  \end{enumerate}

  \item If you used crowdsourcing or conducted research with human subjects, check if you include:
  \begin{enumerate}
    \item The full text of instructions given to participants and screenshots. [Not Applicable]
    \item Descriptions of potential participant risks, with links to Institutional Review Board (IRB) approvals if applicable. [Not Applicable]
    \item The estimated hourly wage paid to participants and the total amount spent on participant compensation. [Not Applicable]
  \end{enumerate}

\end{enumerate}

\onecolumn
\aistatstitle{Supplementary Material:  The Riemannian Geometry Associated to Gradient Flows of Linear Convolutional Networks}

\section{PROOFS AND EXAMPLES}

\subsection{Proofs and Examples of Section \ref{sec:oneDim}}

\subsubsection{Proof of Proposition \ref{prop:correctScaling}}
\begin{proof}
    This proof is inspired by the proof of \cite[Corollary~5.14]{kohn2022geometry}.
    Define $C := \prod_{l=1}^H \Vert w_l \Vert^2$
    and $\beta_l := \Vert \lambda_l w_l \Vert^2 = \lambda_l^2 \Vert w_l \Vert^2$ for all $l = 1, \ldots, H$.
    Then we have reformulated our problem in that we want to show that the conditions  
    \begin{align}\label{eq:scalingConditions}
        \prod_{l=1}^H \beta_l = C \text{ and }
        \beta_{i+1} - \beta_i = \delta_i \text{ for } i = 1, \ldots, H-1
    \end{align}
    have only one solution $\beta_1, \ldots, \beta_H$ where all $\beta_l$ are positive.
    All $\beta_l$ can be computed from $\beta_1$ via $\beta_l = \beta_1 + \sum_{i=1}^{l-1} \delta_i$. 
    Hence, the conditions \eqref{eq:scalingConditions} can be rewritten as a single equation
    \begin{align}\label{eq:singleEquation}
        \beta_1 \cdot (\beta_1 + \delta_1) \cdots (\beta_1 + \delta_1 + \ldots + \delta_{H-1}) = C.
    \end{align}
    Now we choose a sufficiently large $n_1 > 0$ such that $n_l := n_1 + \sum_{i=1}^{l-1} \delta_i \geq 0$ for all $l = 2, \ldots, H$.
    We define the new unknown $x := \beta_1 - n_1$ such that  the equality \eqref{eq:singleEquation} becomes
    \begin{align}
        \label{eq:finalEquation}
        (x+n_1) (x+n_2) \cdots (x+n_H) = C.
    \end{align}
     Defining $n := \min \{ n_1, \ldots, n_H \} \geq 0$, we observe that the left-hand side of \eqref{eq:finalEquation} is positive and monotonically increasing in the range $(-n, +\infty)$.
    Since the left-hand side is zero for $x=-n$ and converges to $+\infty$ for $x \to \infty$ and $C > 0$, the intermediate value theorem tells us that there is precisely one $x$ in the range $(-n, +\infty)$ that yields the equality in \eqref{eq:finalEquation}.
    By definition of $n$, this is the only solution where all $\beta_l = x+n_l$ are positive. 
\end{proof}

\subsubsection{Proof of Theorem \ref{thm:mainStrideLargerOne} and Example}
\begin{proof}
Under the assumption that $s_l > 1$ for all $l = 1, \ldots, H-1$, every point $[v]$ in the smooth locus of $\mathbb{P}\mathcal{M}$ has a unique point in its fiber $\mu_\mathbb{P}^{-1}([v])$ \cite[Section~5]{kohn2024function}.
In other words, for every end-to-end filter $v$ in the smooth part of $\mathcal{M}$, there is a unique filter tuple, up to scaling, that parametrizes~$v$.
Together with Proposition~\ref{prop:correctScaling}, this shows that, for fixed $\delta_1, \ldots, \delta_{H-1}$, every end-to-end filter $v$ in the smooth locus of the neuromanifold $\mathcal{M}$ has a unique filter tuple $\theta = (w_1, \ldots, w_H)$, up to signs, such that $\mu(\theta)=v$ and the invariants in \eqref{eq:InvConvolution} are satisfied. 
The signs do not affect the neural tangent kernel $K_\theta$ since, due to Observation~\ref{obs:K}, $K_\theta$ is the sum of matrices $K_i$ that are homogeneous of degree two in the filters in $\theta$.
\end{proof}

\begin{example}
    \label{ex:running}
    The following network architecture will serve as a running example. We consider a network with two layers whose filters have sizes $k_1 = 3, k_2=2$ and with strides $s_1 = 2, s_2 = 1$.  For ease of notation we denote the first filter by $a=(a_0,a_1,a_2) \in \mathbb{R}^3$ and the second filter by $b=(b_0,b_1) \in \mathbb{R}^2$. 
The filter of the end-to-end convolution is
\begin{align*}
    \mu(a,b) =  (a_0b_0,a_1b_0,a_2b_0 + a_0b_1, a_1b_1,a_2b_1).
\end{align*}
According to \cite[Example~4.12]{kohn2022geometry}, the neuromanifold is
\begin{align*}
    \mathcal{M} = \{ v \in \mathbb{R}^5 \mid v_0 v_3^2 + v_1^2v_4 - v_1v_2v_3 =0, \, v_2^2 - 4v_0v_4 \geq 0 \}.
\end{align*}
Thus, its Zariski closure is the hypersurface in $\mathbb{R}^5$ defined by $v_0 v_3^2 + v_1^2v_4 - v_1v_2v_3 =0$, that is singular at all points of the form $v = (v_0,0,v_2,0,v_4)$.

Now, we consider a smooth point $ v \in \mathcal{M}$ (i.e., $v_1 \neq 0$ or $v_3 \neq 0$).
As in the proof of Theorem~\ref{thm:mainStrideLargerOne}, there is~-- up to scaling~-- only one choice of filters $(a,b)$ such that $\mu(a,b) = v$. 
If both $v_1 \neq 0$ and $v_3 \neq 0$, then this choice is
\begin{align}\label{eq:substByV}
    a = (\frac{v_0}{v_1}, 1, \frac{v_4}{v_3}) \quad \text{ and } \quad
    b = (v_1,v_3).
\end{align}
If precisely one of $v_1$ and $v_3$ is zero, then the preimage filters are $a=(v_2,v_3,v_4)$ and $ b=(0,1)$ for $v_1=0$, or $a=(v_0,v_1,v_2)$ and $ b=(1,0)$ for $v_3=0$.

Next, for a fixed $\delta \in \mathbb{R}$, we compute the possible scalars $\lambda$ such that $\delta = \Vert \frac{1}{\lambda} b  \Vert^2 - \Vert \lambda a \Vert^2$.
We follow the proof of Proposition~\ref{prop:correctScaling}.
For 2-layer networks, we see that \eqref{eq:singleEquation} becomes a quadratic equation in $\lambda^2$:
\begin{align*}
    \lambda^2 \Vert a \Vert^2 \cdot (\lambda^2 \Vert a \Vert^2 + \delta) = \Vert a \Vert^2 \Vert b \Vert^2.
\end{align*}
Its only positive solution is 
$\lambda^2 = \frac{-\delta + \sqrt{\delta^2+4\Vert a \Vert^2 \Vert b \Vert^2}}{2\Vert a \Vert^2 }$.

Theorem~\ref{thm:mainStrideLargerOne} tells us that we can compute 
$K^{(\delta)}(v)$
as $K_{(\lambda a, \frac{1}{\lambda} b)} = \frac{1}{\lambda^2} K_1 + \lambda^2 K_2$, where
\begin{align*} 
    K_1 = \left( \begin{array}{c c c c c}
         b_0^2 & 0 & b_0b_1 & 0 & 0 \\
         0 & b_0^2 & 0 & b_0b_1 & 0 \\
         b_0b_1 & 0 & b_0^2 + b_1^2 & 0 & b_0b_1 \\
         0 & b_0b_1 & 0 & b_1^2 & 0 \\
         0 & 0 & b_0b_1 & 0 & b_1^2
    \end{array} \right)
    \text{ and }
    K_2 = \left( \begin{array}{c c c c c}
         a_0^2 & a_0a_1 & a_0a_2 & 0 & 0 \\
         a_0a_1 & a_1^2 & a_1a_2 & 0 & 0 \\
         a_0a_2 & a_1a_2 & a_0^2 + a_2^2 & a_0a_1 & a_0a_2 \\
         0 & 0 & a_0a_1 & a_1^2 & a_1a_2 \\
         0 & 0 & a_0a_2 & a_1a_2 & a_2^2
    \end{array} \right).
\end{align*}
This only depends on $v$ and $\delta$.
For instance, if we assume that both $v_1\neq 0$ and $v_3\neq 0$, then $K_1$ and $K_2$ can be expressed only in terms of $v$ using the substitution \eqref{eq:substByV}.
Similarly, $\lambda^2$ can be written in terms of $v$ and $\delta$.
In the balanced case, this simplifies to 
$\lambda^2 = \frac{\Vert b \Vert}{\Vert a \Vert} = \sqrt{\frac{v_1^2+v_3^2}{ \frac{v_0^2}{v_1^2} + 1 + \frac{v_4^2}{v_3^2} }}$.
$\hfill \diamondsuit$
\end{example}

\subsubsection{Proof of Lemma \ref{lem:submersion}}
\begin{proof}
	Since $\mu(\theta)$ is a smooth point of $\mathcal{M}$,  the derivative $d_\theta \mu$ is surjective onto $T_{\mu(\theta)} \mathcal{M}$, and its kernel is spanned by $\dot \theta_1 := (-w_1,w_2,0,\ldots,0), \ldots, \dot \theta_{H-1} := (0,\ldots,0,-w_{H-1},w_H)$, see \cite{kohn2024function}.
	Due to $\dim(\Theta_\delta) = \dim(\mathcal{M})$, it is sufficient to show that this kernel intersects the tangent space $T_\theta \Theta_\delta$ only at $0$. 
	For that, we consider a point $\lambda_1 \dot \theta_1 + \ldots + \lambda_{H-1} \dot \theta_{H-1}$ that is contained in $T_\theta \Theta_\delta$.
	Our aim is to show that all $\lambda_i$ must be equal to $0$.
	
	We begin by proving inductively that each $\lambda_i$ is of the form $\lambda_1 \cdot c_i$ for some positive scalar $c_i$.
	This is immediate for $i=1$. 
	For $i=2$, we plug in the point $\sum_j \lambda_j \dot \theta_j$ into the first condition from \eqref{eq:tangentDelta}.
	This yields the equality $(\lambda_1 - \lambda_2) \Vert w_2 \Vert^2 = -\lambda_1 \Vert w_1 \Vert^2 $.
	Therefore, $\lambda_2 = \lambda_1 \cdot c_2$, where $c_2 := \frac{\Vert w_1 \Vert^2 + \Vert w_2 \Vert^2}{\Vert w_2 \Vert^2}>0$.
	Similarly, for $i>2$, the conditions in \eqref{eq:tangentDelta} yield
	$(\lambda_{i-1}-\lambda_i) \Vert w_i \Vert^2 = - \lambda_1 \Vert w_1 \Vert^2$, which can be rewritten as 
	$\lambda_i = \frac{1}{\Vert w_i \Vert^2} (\lambda_{i-1} \Vert w_i \Vert^2 + \lambda_1 \Vert w_1 \Vert^2)$.
	Applying the induction hypothesis, we obtain
	$\lambda_i = \lambda_1 \cdot c_i$, where
	$c_i := \frac{1}{\Vert w_i \Vert^2} (c_{i-1} \Vert w_i \Vert^2 + \Vert w_1 \Vert^2)>0$.
	
	On the other hand, the last condition in \eqref{eq:tangentDelta} gives us that $\lambda_{H-1} \Vert w_H \Vert^2 = - \lambda_1 \Vert w_1 \Vert^2$.
	Thus, we have $\lambda_1 \cdot c_{H-1} = \lambda_{H-1} = - \lambda_1 \cdot \frac{\Vert w_1 \Vert^2}{\Vert w_H \Vert^2}$, which shows that $\lambda_1$ must be $0$, and so all $\lambda_i$ are $0$.
\end{proof}

\subsubsection{Proof of Corollary \ref{cor:RiemannianMetric} and Examples}

\begin{proof}
	The first part of the claim follows immediately from Lemma~\ref{lem:submersion} and Theorem~\ref{thm:mainStrideLargerOne}.
	More concretely, Lemma~\ref{lem:submersion} states that the derivatives of $\mu_\delta := \mu |_{\Theta_\delta}$ are bijective.
	For such a submersion with finite fibers, the pushforward metric can be computed via averaging over the metrics on the fibers:	\begin{align*}
		\Vert \dot v \Vert_{\mu_\delta}^2 
		:= & \frac{1}{|\mu_\delta^{-1}(v)|} \sum_{\theta \in \mu_\delta^{-1}(v)} \Vert (d_\theta \mu_\delta)^{-1} (\dot v) \Vert^2
	\\	= & \frac{1}{|\mu_\delta^{-1}(v)|} \sum_{\theta \in \mu_\delta^{-1}(v)} ((d_\theta \mu_\delta)^{-1} (\dot v))^\top (d_\theta \mu_\delta)^{-1} (\dot v)
		\\ = & \frac{1}{|\mu_\delta^{-1}(v)|}  \sum_{\theta \in \mu_\delta^{-1}(v)}  \dot v^\top (J_\theta J_\theta^\top)^{-1} \dot v 
		=\frac{1}{|\mu_\delta^{-1}(v)|}  \sum_{\theta \in \mu_\delta^{-1}(v)}  \dot v^\top K_\theta^{-1} \dot v,
	\end{align*}
	where $\dot v \in T_v \mathcal{M}$ for some $v \in \mathrm{Reg}(\mathcal{M})$.
	By	Theorem~\ref{thm:mainStrideLargerOne},  all $\theta \in \mu_\delta^{-1}(v)$  have the same neural tangent kernel $K_\theta = K^{(\delta)}(v)$,
	and so we conclude
	that the pushforward metric is 
	$\Vert \dot v \Vert_{\mu_\delta}^2 = \dot v^\top (K^{(\delta)}(v))^{-1} \dot v$.
\end{proof}

\begin{example}\label{ex:singular}
We consider the network architecture from Example~\ref{ex:running}.
Recall that the singular points of the neuromanifold $\mathcal{M}$ are of the form $v = (v_0,0,v_2,0,v_4)$.
They are parametrized by filters of the form $a = (a_0,0,a_2)$ and $b=(b_0,b_1)$.

    If $(a_0,a_2)$ and $(b_0,b_1)$ are linearly independent, 
    then the resulting end-to-end filter $v$ has two distinct (up to scaling) parametrizations, namely
    $v = \mu(a_0,0,a_2,b_0,b_1) = \mu(b_0,0,b_1,a_0,a_2)$.
    At both parameter tuples, the derivative of $\mu$ is of maximal rank $4$, and so is the neural tangent kernel.
    However, in general, $K_{(a_0,0,a_2,b_0,b_1)} \neq K_{(b_0,0,b_1,a_0,a_2)}$, even for fixed $\delta$-invariant.
    One concrete balanced (i.e., $\delta=0$) example is $(a_0,a_2)=(1,2)$ and $(b_0,b_1)=(2,1)$. In this case, 
    \begin{align*}
        K_{(a_0,0,a_2,b_0,b_1)} = \left( \begin{array}{c c c c c}
         5 & 0 & 4 & 0 & 0 \\
         0 & 4 & 0 & 2 & 0 \\
         4 & 0 & 10 & 0 & 4 \\
         0 & 2 & 0 & 1 & 0 \\
         0 & 0 & 4 & 0 & 5
    \end{array} \right)  \text{ and } 
        K_{(b_0,0,b_1,a_0,a_2)} = \left( \begin{array}{c c c c c}
         5 & 0 & 4 & 0 & 0 \\
         0 & 1 & 0 & 2 & 0 \\
         4 & 0 & 10 & 0 & 4 \\
         0 & 2 & 0 & 4 & 0 \\
         0 & 0 & 4 & 0 & 5
    \end{array} \right)
    \end{align*}

    If $(a_0,a_2)$ and $(b_0,b_1)$ are linearly dependent, they are the same up to scaling. This means that, up to scaling, $v$ has only one preimage under $\mu$.
    That preimage  $(a_0,0,a_2,b_0,b_1)$ is a critical point of the map $\mu$, i.e., the rank of the neural tangent kernel $K_{(a_0,0,a_2,b_0,b_1)}$ is less than the generic rank $4$.
    $\hfill \diamondsuit$
\end{example}

\begin{example}\label{ex:varyingDelta}
We consider the same network architecture as in Example~\ref{ex:running}.
Suppose $v \in \mathcal{M}$ is such that $v_0$ and $v_1$ are not equal to 0. 
Such a $v$ is in particular a smooth point of the neuromanifold $\mathcal{M}$.
We show that for different $\delta$ we obtain different $K^{(\delta)}(v)$.

We fix filters $a$ and $b$ such that  $\mu(a,b) = v$. 
Our assumptions on $v$ imply that $a_0 \neq 0$ and $a_1 \neq 0$. 
By Example~\ref{ex:running}, 
we have $K^{(\delta)}(v) = K_{(\lambda a, \frac{1}{\lambda} b)}$, 
where 
$\lambda^2 = \frac{-\delta + \sqrt{\delta^2+4\Vert a \Vert^2 \Vert b \Vert^2}}{2\Vert a \Vert^2 }$. 
In particular, the entry of $K^{(\delta)}(v)$ at position $(2,1)$ is equal to $\lambda^2 a_0 a_1 \neq 0$.
The function which to $\delta \in \mathbb{R}$ associates $\frac{-\delta + \sqrt{\delta^2+4\Vert a \Vert^2 \Vert b \Vert^2}}{2\Vert a \Vert^2 }$ is strictly decreasing and therefore bijective.
Hence, distinct  $\delta$'s yield distinct values for $\lambda^2$ and thus for the $(2,1)$-entry of $K^{(\delta)}(v)$.
$\hfill \diamondsuit$
\end{example}

\subsubsection{Proof of Proposition \ref{prop:exlude0} and Example}

To prove Proposition~\ref{prop:exlude0}, we start by rewriting the squared-error loss in terms of filters instead of  convolutions.
Under the assumption $N \geq k$, it is shown in \cite[Corollary~7.2]{kohn2024function}
that, for almost all input training data $X_1, \ldots, X_N \in \mathbb{R}^{d_0}$, 
minimizing $\ell_\mathcal{D}(\alpha)$ is equivalent to minimizing
\begin{align*}
    \ell_{A,u}(w) := (w-u)^\top A (w-u),
\end{align*}
where $w$ is the filter of the convolution $\alpha$, 
$u$ is a data filter that depends on the training data $\mathcal{D}$,
and $A$ is a symmetric positive-definite matrix that depends on the input training data $X_1,\ldots,X_N$.
We denote the corresponding loss in parameter space by $\mathcal{L}_{A,u} := \ell_{A,u} \circ \mu$.

When the data $\mathcal{D}$ is generic, the vector $Au$ is generic as well (see \cite[Section~7]{kohn2024function}). 
We use this to prove the  technical Lemma~\ref{lem:technical}, that is based on the following standard construction from algebraic geometry. 
Given a $k$-dimensional vector space $V$, its projectivization $\mathbb{P}(V)$ is a $(k-1)$-dimensional projective space.
The projectivization $\mathbb{P}(V^\ast)$ of the dual vector space $V^\ast$ consists of all the hyperplanes in  $\mathbb{P}(V)$.
Given a subvariety $X \subseteq \mathbb{P}(V)$, we say that a hyperplane $H \subseteq \mathbb{P}(V)$ is tangent to $X$ if there is a smooth point $x \in X$ such that $H$ contains the embedded tangent space $T_x X \subseteq \mathbb{P}(V)$.
Thinking of the hyperplanes in $\mathbb{P}(V)$ as points of the dual projective space $\mathbb{P}(V^\ast)$, the set of hyperplanes that are tangent to $X$ is a subset of $\mathbb{P}(V^\ast)$.
Its Zariski closure is a proper subvariety of $\mathbb{P}(V^\ast)$, called the \emph{dual variety} of $X$~[Chapter~1]\cite{gkz}.

\begin{lemma}
\label{lem:technical}
    Let $\theta = (w_1, \ldots, w_H)$ with $w_i \in \mathbb{R}^{k_i} \setminus \{ 0 \}$, and let $N \geq k$.
    For almost all $N$-tuples of data $\mathcal{D}$, the vector $Au$ is not orthogonal to the image of the derivative $d_\theta \mu$.
\end{lemma}

\begin{proof}
    Let us first consider the case that $\mu(\theta)$ is a smooth point of the neuromanifold $\mathcal{M}$.
    Then, the image of $d_\theta \mu$ is the tangent space of $\mathcal{M}$ at $\mu(\theta)$.
    A vector $v=(v_0, \ldots, v_{k-1})$ is orthogonal to this tangent space if and only if the hyperplane 
    $P_v := \{ (h_0, \ldots, h_{k-1}) \mid \sum_i v_i h_i=0 \}$ contains the tangent space.
    Due to the genericity of the data $\mathcal{D}$, the vector $Au$ and the hyperplane $P_{Au}$ are generic. 
    Since the dual variety of the Zariski closure $\mathbb{P}\overline{\mathcal{M}}$ of the projective neuromanifold $\mathbb{P}\mathcal{M}$ is a proper subvariety of $(\mathbb{P}^{k-1})^\ast$, the generic hyperplane $P_{Au}$, considered as a point in the dual space $(\mathbb{P}^{k-1})^\ast$, is not contained in the dual variety of $\mathbb{P}\overline{\mathcal{M}}$.
    Hence, the hyperplane $P_{Au}$ is not tangent to $\mathbb{P}\overline{\mathcal{M}}$, and so the vector $Au$ is not orthogonal 
    to $T_{\mu(\theta)} \mathcal{M} = \mathrm{im}(d_\theta \mu)$.
    
    Now, we consider  arbitrary $\theta$ where all layer filters are nonzero.
    We denote by $r$ the rank of the derivative $d_\theta \mu$.
    Since $d_\theta \mu (\theta) = H \cdot  \mu(\theta) \neq 0$, we have that $r \geq 1$.
    We further denote by $X_r$ the set of all filter tuples $\theta'$ such that the rank of $d_{\theta'}\mu$ is $r$, and by $Y_r$ the Zariski closure of $\mu(X_r)$.
    Note that $Y_{\dim \mathcal{M}} = \overline{\mathcal{M}}$.
    If $\mu(\theta)$ is a smooth point of $Y_r$, then $\mathrm{im}(d_\theta \mu)$ is the tangent space $T_{\mu(\theta)} Y_r$.
    Otherwise, since $\mathrm{im}(d_\theta \mu)$ has dimension $r$, this linear space is contained in the Zariski closure of the set of all tangent spaces at smooth points of $Y_r$.
    Therefore, the dual variety of $\mathbb{P} Y_r$ contains all hyperplanes that contain the $r$-dimensional $\mathrm{im}(d_\theta \mu)$.
    Since, as above, the hyperplane $P_{Au}$ is a generic point in $(\mathbb{P}^{k-1})^\ast$, it is not contained in the dual variety of $\mathbb{P} Y_r$.
    Thus, it does not contain $\mathrm{im}(d_\theta \mu)$, and so $Au$ is not orthogonal to $\mathrm{im}(d_\theta \mu)$.
\end{proof}

\begin{proof}[Proof of Proposition \ref{prop:exlude0}]
    For generic weight initialization, the $\delta_i$ (for $i = 1, \ldots, H-1$) are generic as well, and also the $\sum_{i=k}^{l} \delta_i$.
    Since the invariants \eqref{eq:InvConvolution} stay constant during gradient flow, at the point $\theta = (w_1, \ldots, w_H)$ of convergence, at most one of the filters $w_l$ can be zero. Indeed, if two filters $w_k$ and $w_l$ are equal to zero then $\sum_{i=k}^{l} \delta_i$ has to be zero, which is not the case because of genericity. 
    If none of the filters is zero, we are done. 
    Therefore, we assume now that precisely one of the filters, say $w_l$, is zero.

    We now consider the Hessian matrix $\mathcal{H}$ at $\theta$ of the loss function $\mathcal{L}_{A,u}$ that is equivalent to the squared-error loss.
    The Hessian matrix consists of blocks, each corresponding to a pair of filters $(w_i,w_j)$.
    We denote these blocks by $\mathcal{H}_{(w_i,w_j)}$.
    Those blocks encode bilinear functions 
    \begin{align*}
        d^2_{w_i,w_j} \mathcal{L}_{A,u}: (\dot w_i, \dot w_j) \mapsto 
        2 \left( \mu_i(\dot w_i)^\top A \mu_j (\dot w_j) - \mu_{i,j}(\dot w_i, \dot w_j)^\top A u \right),
    \end{align*}
    where $\mu_i(\dot w_i) = \mu(w_1, \ldots, w_{i-1}, \dot w_i, w_{i+1}, \ldots, w_H)$, $\mu_j$ is analogously defined, and similarly $\mu_{i,j}$ substitutes both $w_i$ and $w_j$ by $\dot w_i$ and $\dot w_j$ (if $i=j$, then $\mu_{i,i}=0$).
    In the case that neither $i$ nor $j$ are equal to $l$, then our assumption $w_l=0$ implies that $d^2_{w_i,w_j} \mathcal{L}_{A,u}=0$. Thus, also the corresponding Hessian block $\mathcal{H}_{(w_i,w_j)}$ is zero. 

    Now we show that not all Hessian blocks with $j=l$ and $i \neq l$ are zero. 
    Those blocks encode bilinear forms  $d^2_{w_i,w_l} \mathcal{L}_{A,u} = -2 \mu_{i,l}^\top Au$. 
    We assume for contradiction that all those forms are zero.
    We pick a nonzero filter $c \in \mathbb{R}^{k_l}$ and consider the map $\tilde{\mu}: (w'_1, \ldots, w'_{l-1}, w'_{l+1}, \ldots, w'_H) \mapsto \mu(w'_1, \ldots, w'_{l-1}, c, w'_{l+1}, \ldots, w'_H)$.
    We also write $\tilde{\theta}$ for the filter tuple that is obtained from $\theta$ by omitting $w_l$.
    Then, $d_{\tilde{\theta}} \tilde{\mu}
    (\dot w_1, \ldots, \dot w_{l-1}, \dot w_{l+1}, \ldots, \dot w_H)^\top Au =   \sum_{i \neq l} \mu_{i,l}(\dot w_i,c)^\top Au = 0$ for all $\dot w_i \in \mathbb{R}^{k_i}$.
    This means that the vector $Au$ is orthogonal to the image of the derivative $d_{\tilde{\theta}} \tilde{\mu}$.
    However, this contradicts Lemma~\ref{lem:technical} applied to the $(H-1)$-layer network parametrization map $\tilde{\mu}$.

    Therefore, one of the blocks $\mathcal{H}_{(w_i, w_l)}$ with $i \neq l$ has a non-zero entry.
    We denote the coordinates of that entry in the Hessian matrix $\mathcal{H}$ by $(s,t)$ (such that the index $t$ corresponds to the filter $w_l$).
    We consider the $s$-th and $t$-th standard basis vectors $e_s$ and $e_t$ of $\mathbb{R}^{\sum_{j=1}^H k_j}$, and the vector $v = \alpha e_s + e_t$ for some scalar $\alpha \in \mathbb{R}$.
    We have 
    \begin{align*}
    v^\top \mathcal{H} v = \alpha^2 \mathcal{H}_{ss} + 2 \alpha \mathcal{H}_{st} + \mathcal{H}_{tt} 
    = 2 \alpha \mathcal{H}_{st} + \mathcal{H}_{tt}.
\end{align*}
Since $\mathcal{H}_{st} \neq 0$, we can choose $\alpha$ such that $v^\top \mathcal{H} v < 0$. 
 Hence, $\theta$ is a critical point for which the Hessian has at least one negative eigenvalue. This is called a \emph{strict saddle},
and we know from \cite{bah2022learning} that, for almost all initializations, gradient flow avoids strict saddles. Thus, $\theta$ with one of the layer filters equal to zero could not have been the point of convergence.
\end{proof}

\begin{example}\label{ex:somestridesequaltoone}
We change the previous running example to both layers having stride  one.
Given two layers $a=(a_0,a_1,a_2)$ and $b=(b_0,b_1)$, the end-to-end filter is
$\mu(a,b) = (a_0b_0, a_0b_1+a_1b_0, a_1b_1+a_2b_0, a_2b_1)$.
It corresponds to a cubic polynomial 
$a_0b_0x^3 + (a_0b_1+a_1b_0)x^2y + (a_1b_1+a_2b_0)xy^2 + a_2b_1y^3$ 
that factors into a quadratic term $a_0 x^2 + a_1 xy + a_2 y^2$ and a linear one $b_0 x + b_1 y$.
Every cubic polynomial has such a factorization, so the neuromanifold is $\mathcal{M} = \mathbb{R}^4$.
Cubic polynomials with only one real root have precisely one such factorization (up to scaling), while cubics with three distinct real roots have three distinct factorizations (up to scaling).

For instance, the filter $v = (0,1,1,0)$ that corresponds to the cubic $x^2y + xy^2 = xy(x+y)$ has the following factorizations according to the network architecture:
\begin{enumerate}
    \item  $a = \sqrt[4]{2} \cdot (0,1,0)$ and $b= \frac{1}{\sqrt[4]{2}} \cdot (1,1)$
    \item $a = \frac{1}{\sqrt[4]{2}} \cdot(1,1,0)$ and $b=\sqrt[4]{2} \cdot(0,1)$
    \item $a = \frac{1}{\sqrt[4]{2}} \cdot(0,1,1)$ and $b = \sqrt[4]{2} \cdot(1,0)$
\end{enumerate}
Each of these filter pairs has been scaled such that $\Vert b \Vert^2 = \Vert a \Vert^2$, i.e., $\delta = 0$.

The neural tangent kernel for this network architecture is
\begin{align*} 
    K_{a,b} = \left(
\begin{array}{cccc}
   a_0^2+b_0^2  & a_0 a_1 +  b_0 b_1 & a_0a_2 & 0 \\
   a_0 a_1 +  b_0 b_1  & a_0^2+a_1^2+b_0^2+b_1^2 & a_0a_1 + a_1a_2 + b_0b_1 & a_0a_2
   \\ 
   a_0a_2 & a_0a_1 + a_1a_2 + b_0b_1 & a_1^2+a_2^2+b_0^2+b_1^2 & a_1 a_2 + b_0b_1
   \\ 0 & a_0a_2 & a_1 a_2 + b_0b_1 & a_2^2+b_1^2
\end{array}
    \right).
\end{align*}
For the three factorizations of $v = (0,1,1,0)$ above, the matrices $K_{a,b}$ are
\begin{align*}
  \frac{1}{\sqrt{2}}  \left( \begin{array}{cccc}
        1 & 1 & 0 & 0 \\
        1 & 4 & 1 & 0 \\
        0 & 1 & 4 & 1 \\
        0 & 0 & 1 & 1
    \end{array} \right), \quad 
    \frac{1}{\sqrt{2}}  \left( \begin{array}{cccc}
        1 & 1 & 0 & 0 \\
        1 & 4 & 1 & 0 \\
        0 & 1 & 3 & 0 \\
        0 & 0 & 0 & 2
    \end{array} \right), \quad \text{and} \quad 
    \frac{1}{\sqrt{2}}  \left( \begin{array}{cccc}
        2&0&0&0 \\
        0&3&1&0 \\
        0&1&4&1 \\
        0&0&1&1
    \end{array} \right),
\end{align*}
showing that the neural tangent kernel does not only depend on the end-to-end filter $v$ and the invariant $\delta$.
 $\hfill \diamondsuit$
\end{example}

\subsection{Proofs of Section \ref{sec:higherDim}}

\subsubsection{Proof of Lemma \ref{lem:polynomialMultiplication}}
\begin{proof}
    We abbreviate \eqref{eq:higherDimensionalConvolution} by $(\alpha_{w,\boldsymbol{s}}(X))_{\boldsymbol{i}} = \sum_{\boldsymbol{j}=\boldsymbol{0}}^{\boldsymbol{k}-\boldsymbol{1}} w_{\boldsymbol{j}} X_{\boldsymbol{i}\boldsymbol{s}+\boldsymbol{j}}$.
    The assertion follows inductively from composing two such convolutions:
    \begin{align*}       (\alpha_{w_2,\boldsymbol{s}_2}\circ \alpha_{w_1,\boldsymbol{s}_1})(X)_{\boldsymbol{i}} &= \sum_{\boldsymbol{j}=\boldsymbol{0}}^{\boldsymbol{k}_2-\boldsymbol{1}} w_{2,\boldsymbol{j}} \alpha_{w_1,\boldsymbol{s}_1}(X)_{\boldsymbol{i} \boldsymbol{s}_2 + \boldsymbol{j}} 
        = \sum_{\boldsymbol{j}=\boldsymbol{0}}^{\boldsymbol{k}_2-\boldsymbol{1}} w_{2,\boldsymbol{j}}  \sum_{\boldsymbol{l}=\boldsymbol{0}}^{\boldsymbol{k}_1-\boldsymbol{1}} w_{1,\boldsymbol{l}} 
        X_{(\boldsymbol{i} \boldsymbol{s}_2 + \boldsymbol{j})\boldsymbol{s}_1+\boldsymbol{l}} \\
        &= \sum_{\boldsymbol{m}=0}^{(\boldsymbol{k}_2-\boldsymbol{1})\boldsymbol{s}_1+\boldsymbol{k}_1-\boldsymbol{1}}
        v_{\boldsymbol{m}} X_{\boldsymbol{i} \boldsymbol{s}_2\boldsymbol{s}_1+\boldsymbol{m}},
    \end{align*}
    where  
    $ I_{\boldsymbol{m}} := \{ (\boldsymbol{j},\boldsymbol{l}) \mid \boldsymbol{0} \leq \boldsymbol{j} < \boldsymbol{k}_2, \; \boldsymbol{0} \leq \boldsymbol{l} < \boldsymbol{k}_1,  \;   \boldsymbol{j}\boldsymbol{s}_1+\boldsymbol{l}=\boldsymbol{m} \}$ and $v_{\boldsymbol{m}} := \sum_{(\boldsymbol{j},\boldsymbol{l}) \in I_{\boldsymbol{m}}} w_{2,\boldsymbol{j}}w_{1,\boldsymbol{l}}$.
    Note that $v_{\boldsymbol{m}}$ coincides with the coefficient of the polynomial $\pi_{\boldsymbol{s}_1}(w_2)\pi_{\boldsymbol{1}}(w_1)$, where the degrees of the variables $y_1, \ldots, y_D$ are $m_1, \ldots, m_D$.
\end{proof}

\subsubsection{Proof of Corollary \ref{cor:birationalHighDimensions}}
\begin{proof}
    Our assumption means that, for every layer $l$, at least two variable pairs $(x_{m_1},y_{m_1})$ and $(x_{m_2},y_{m_2})$ appear in the polynomials in $\mathbb{R}[(\boldsymbol{x},\boldsymbol{y})^{\boldsymbol{t}_l}]_{\boldsymbol{k}_l-\boldsymbol{1}}$ that correspond to the $l$-th layer filter.
    Hence, the polynomials in $\mathbb{R}[(\boldsymbol{x},\boldsymbol{y})^{\boldsymbol{t}_l}]_{\boldsymbol{k}_l-\boldsymbol{1}}$ are generically irreducible. (This is a standard fact from algebraic geometry, following from Bertini's theorem; see the proof of \cite[Corollary~4.16]{kohn2022geometry} for a formal argument.)
    Therefore, for a generic filter tuple $\theta=(w_1, \ldots, w_H)$, the unordered set of filters  $\{w_1, \ldots, w_H \}$ can be recovered (up to scalars) from the end-to-end filter $\mu(\theta)$ by computing the factorization of $\pi_{\boldsymbol{1}}(\mu(\theta))$ into its irreducible factors $\{ \pi_{\boldsymbol{t}_1}(w_1), \ldots, \pi_{\boldsymbol{t}_H}(w_H) \}$.
\end{proof}

\subsubsection{Proof of Proposition \ref{prop:correctScalingHighDimension}}
\begin{proof}
    The same proof as for Proposition~\ref{prop:correctScaling} applies.
\end{proof}

\subsubsection{Proof of Theorem \ref{thm:mainHigherDimension}}
\begin{proof}
    We see from Lemma~\ref{lem:polynomialMultiplication} that every fiber of the projectivization $\mu_\mathbb{P}$ of $\mu$ is finite.
    Moreover, Corollary~\ref{cor:birationalHighDimensions} says that, after possibly modding out some permutations of the input layers (cf. Observation~\ref{obs:swap}), the map $\mu_\mathbb{P}$ becomes birational (i.e., its generic fiber is a singleton).
    For such a map, a standard fact from algebraic geometry~\cite[Lemma~3.2]{kohn2017secants} tells us that the image of $\mu_\mathbb{P}$ (over the complex numbers) is smooth precisely at those points that have a single preimage under $\mu_\mathbb{P}$ (modulo the necessary layer permutations) and where the derivative of $\mu_\mathbb{P}$ at that preimage has maximal rank.  
    In other words,
    for every end-to-end filter $v$ in the smooth locus of $\mathcal{M}$, there is a unique filter tuple $\theta$, up to scalars and possibly some layer permutations, that parametrizes $v$; moreover, $d_\theta\mu$ has maximal rank.
    By Proposition~\ref{prop:correctScalingHighDimension}, the invariants $\delta$ fix the scalars in the filter tuple $\theta$, up to signs. 
    The signs and the possible layer permutations do not affect the neural tangent kernel $K_\theta$, by Observations~\ref{obs:K} and~\ref{obs:swap}.
    Finally, the proofs of Lemma~\ref{lem:submersion} and Corollary~\ref{cor:RiemannianMetric} hold verbatim.
\end{proof}

\subsection{Example and Proof of Section \ref{sec:fullyConn}}
\begin{example}\label{ex:fullyconnected}
    For $\Delta_1=diag(0,15/4)$, the neural tangent kernel is parameter-dependent.
    Indeed, let the end-to-end product matrix be equal to the identity $W=I$. Consider $V_2=\operatorname{diag}(1,2)$ and $V_1=\operatorname{diag}(1,1/2)$. Hence $V_2V_1=I_2=W$ and the balancedness constant matrix is indeed equal to $V_2^TV_2 - V_1 V_1^T = \operatorname{diag}(0,15/4) = \Delta_1$.

Consider then $U_2 = \left( \begin{array}{c c}
    0 & 2 \\
    1 & 0
\end{array} \right)$ and $U_1 = \left(\begin{array}{c c}
    0 & 1 \\
    1/2 & 0
\end{array} \right)$, we have again $U_2U_1=I_2$ and $U_2^TU_2 - U_1 U_1^T = \operatorname{diag}(0,15/4)=\Delta_1$. However, using the formula \eqref{NTK-fully-general} for the NTK in this case, we have $ K^{(\Delta_1)}_{V_2,V_1} = \nabla_W \ell(W) V_1^T V_1 + V_2 V_2^T \nabla_W \ell(W)  \neq \nabla_W \ell(W) U_1^T  U_1 + U_2 U_2^T \nabla_W \ell(W) = K^{(\Delta_1)}_{U_2,U_1}$. 
Hence the neural tangent kernel is parameter-dependent. Using Remark \ref{rk:map algebraicity}, we conclude that this is the case for almost all $\Delta_1$. We can also easily generalize this example to any depth and width of the network.
\end{example}

\subsubsection{Proof of Theorem \ref{thm:fullyconnected}}
\begin{proof}
    Consider the map $\mu: GL_d(\mathbb{R})^H \to GL_d(\mathbb{R})$ defined by $\mu(W_H,\ldots,W_1) = W_H \cdots W_1$. Then the fibers of a generic $W$ can be written as $$(W_HG_{H-1},G_{H-1}^{-1}W_{H-1}G_{H-2},\ldots,G_{2}^{-1}W_{2}G_{1},G_{1}^{-1}W_{1}),$$ where $(W_H, \ldots, W_1)$ are any fixed matrices such that their product is equal to $W$.
Let us prove that the fiber such that all the layers are balanced leads then to all $G_i$ being orthogonal. It was shown in \cite{arora2018optimization} that we have for all balanced weights $(W'_H,\ldots,W'_1)$ in the fiber of $W$, for all $j \in \{1,\ldots,H\}$ the following equations $$ W'_H \cdots W'_j (W'_H \cdots W'_j)^T = (WW^T)^{\frac{H-j+1}{H}}.$$
Let $(W_H,\ldots,W_1)$ be balanced matrices in the fiber of $W$, and $G_1,\ldots,G_{H-1}$ invertible matrices such that $(W_HG_{H-1},G_{H-1}^{-1}W_{H-1}G_{H-2},\ldots,G_{2}^{-1}W_{2}G_{1},G_{1}^{-1}W_{1})$ are also balanced matrices (in the fiber of $W$).
We have, for  $j \in \{1,\ldots,H-1\}$,
\begin{align*}
    W_H \cdots W_{j+1} G_j G_j^T (W_H \cdots W_{j+1})^T & = (WW^T)^{\frac{H-j}{H}} \\ 
    \text{ and } \quad\quad W_H \cdots W_{j+1} (W_H \cdots W_{j+1})^T & =(WW^T)^{\frac{H-j}{H}}.
\end{align*}
Hence by unicity of the symmetric PSD root of a symmetric PSD matrix we have \begin{align*}
    W_H \cdots W_{j+1} G_j G_j^T (W_H \cdots W_{j+1})^T = W_H \cdots W_{j+1} (W_H \cdots W_{j+1})^T.
\end{align*} 
Therefore, since all weight matrices are invertible we obtain $G_jG_j^T=I$ so that the $G_j$, $j=1,\hdots,H-1$ are orthogonal. Conversely if some weight layers are balanced and in the fiber of $W$, then one can easily check that for $G_j$ orthogonal, the layers are also balanced.
\end{proof}

\end{document}